\documentclass{article}

\usepackage{microtype}
\usepackage{graphicx}
\usepackage{subfigure}
\usepackage{booktabs} 

\usepackage{hyperref}

\usepackage{icml2025}

\usepackage{amsmath}
\usepackage{amssymb}
\usepackage{mathtools}
\usepackage{amsthm}
\usepackage[capitalize,noabbrev]{cleveref}
\usepackage{multirow}
\theoremstyle{plain}
\newtheorem{theorem}{Theorem}[section]

\newtheorem{corollary}[theorem]{Corollary}
\theoremstyle{definition}
\newtheorem{definition}[theorem]{Definition}

\theoremstyle{remark}
\newtheorem{remark}[theorem]{Remark}
\newcommand{\bm}[1]{\mathbf{#1}}

\usepackage[textsize=tiny]{todonotes}

\icmltitlerunning{AWP: Activation-Aware Weight Pruning and Quantization with Projected Gradient Descent}

\begin{document}

\twocolumn[
\icmltitle{AWP: Activation-Aware Weight Pruning and Quantization\\ with Projected Gradient Descent}

\icmlsetsymbol{equal}{*}

\begin{icmlauthorlist}
\icmlauthor{Jing Liu}{merl}
\icmlauthor{Toshiaki Koike-Akino}{merl}
\icmlauthor{Ye Wang}{merl}
\icmlauthor{Hassan Mansour}{merl}
\icmlauthor{Matthew Brand}{merl}
\end{icmlauthorlist}

\icmlaffiliation{merl}{Mitsubishi Electric Research Laboratories (MERL), 201 Broadway, Cambridge, MA 02139, USA.}

\icmlcorrespondingauthor{Jing Liu}{jiliu@merl.com}

\icmlkeywords{Machine Learning, Mixture of Experts, Dynamic Network, Activation-Aware Pruning, LLM}

\vskip 0.3in
]
\printAffiliationsAndNotice{}  %

\begin{abstract}

To address the enormous size of Large Language Models (LLMs), model compression methods, such as quantization and pruning, are often deployed, especially on edge devices. In this work, we focus on layer-wise post-training quantization and pruning. Drawing connections between activation-aware weight pruning and sparse approximation problems, and motivated by the success of Iterative Hard Thresholding (IHT), we propose a unified method for \textbf{A}ctivation-aware \textbf{W}eight pruning and quantization via \textbf{P}rojected gradient descent (AWP). Our experiments demonstrate that AWP outperforms state-of-the-art LLM pruning and quantization methods. Theoretical convergence guarantees of the proposed method for pruning are also provided.
\end{abstract}

\section{Introduction }

Large transformer-based models have demonstrated superior performance in many tasks, including natural language processing and computer vision. However, their enormous size poses significant challenges for efficient deployment especially on the edge devices. Quantization and pruning are promising approaches for model compression. For example, it was found\footnote{{https://github.com/ggml-org/llama.cpp/pull/1684}} that the 4-bit quantized Llama2-13B model achieves better perplexity, smaller size, and faster inference speed than the unquantized Llama2-7B model.

In this work, we consider Post-Training Quantization (PTQ) and pruning/sparsification. This is in contrast to Quantization-Aware Training (QAT) and sparse network training of large models which are still resource intensive, and many applications only consider leveraging pre-trained foundation models instead of training a model from scratch. In the post-training compression setup~\cite{nagel2020up,DBLP:journals/corr/abs-2102-05426,pmlr-v139-hubara21a,LIANG2021370}, a trained but uncompressed model is given, together with a small amount of calibration data, with the aim to produce an accurate compressed model without retraining. As the foundation models are very large, which may exceed the hardware memory limit during compression, recent works, such as~\citet{pmlr-v119-wang20c,nagel2020up,pmlr-v139-hubara21a}, further break the compression task into layer-wise sub-problems, identifying a compressed weight approximation for each layer, given a sub-sample of the layer’s inputs and outputs based on calibration data. This approach is the focus of our work.

We introduce a new compression framework based on projected gradient descent (PGD), inspired by compressive sensing and sparse approximation methods.
The contributions of our paper are listed below:
\vspace{-6pt}
\begin{itemize}
  \setlength{\itemsep}{2pt}
  \setlength{\parskip}{0pt}
  \setlength{\parsep}{0pt}
  
\item We pose the LLM compression problem as a sparse approximation problem.

\item We show that PGD is a viable solution to unify pruning and quantization problems without requiring computationally-intensive operations such as SVD.

\item The proposed method outperforms state-of-the-art LLM compression methods on several benchmarks.

\item We provide theoretical convergence guarantees for the proposed method for pruning (details can be found in Appendix).
\end{itemize}

\section{Related Work and Background}

For the post-training quantization and pruning, activation-aware methods, e.g.,~\citet{lin2024awq, frantar2022gptq, wang2024q,frantar2023sparsegpt,sun2023wanda,10.5555/3600270.3600593,ZHANG2025101778}, demonstrate superior performance as they consider the input activation statistics to guide the quantization and pruning. For example, the traditional magnitude based weight pruning essentially solves a sparse approximation problem, without considering input activations:
\begin{equation}
\min_{\mathbf{W}_{\text{sparse}} \in \mathcal{C}_{\text{sparse}}}
\big [ \mathcal{L}(\mathbf{W}_\mathrm{sparse})
:= \|\mathbf{W} - \mathbf{W}_{\text{sparse}} \|_\mathrm{F}^2 \big ],
\end{equation}
where $\mathbf{W} \in \mathbb{R}^{d_\text{out} \times d_\text{in}}$ is the original uncompressed model weight, $\mathbf{W}_{\text{sparse}} \in \mathbb{R}^{d_\text{out} \times d_\text{in}}$ is the pruned/sparsified weight, and $\mathcal{C}_{\text{sparse}}$ is a constraint set, which could be the ratio of zero elements in $\mathbf{W}_{\text{sparse}}$ and/or its sparsity patterns.

For transformer-based models,
the linear layer takes in input activations $\mathbf{X} \in \mathbb{R}^{d_\text{in} \times n}$, where $n$ is a total token length given calibration data.
In contrast, activation-aware pruning changes the minimization objective to
\begin{align} \label{eq:activation_aware_X}
\mathcal{L}'(\mathbf{W}_\mathrm{sparse})
&:= \|\mathbf{W}\mathbf{X}  - \mathbf{W}_{\text{sparse}} \mathbf{X} \|_\mathrm{F}^2 \\
&\ =
\|\mathbf{W}\mathbf{C}^{\frac{1}{2}}  - \mathbf{W}_{\text{sparse}} \mathbf{C}^{\frac{1}{2}} \|_\mathrm{F}^2 ,\label{eq:activation_aware_C}
\end{align}
where $\bm C=\bm {X}\bm {X}^\top\in\mathbb{R}^{d_\text{in}\times d_\text{in}}$ is the auto-correlation of the input activation, and  $\mathbf{C}^{\frac{1}{2}}$ is the matrix square root of $\bm{C}$.

However, there is no closed-form solution to the above problem and several heuristic methods have been proposed.
Wanda~\cite{sun2023wanda} computes the $\ell_2$-norm of each row of $\mathbf{X}$, i.e., $\| \mathbf{X}[i,:] \|_2$, for $i=1,\ldots,d_\text{in}$,  to scale the corresponding column of $\bm W$, and then performs magnitude based pruning on this scaled matrix to obtain the pruning mask for the weight $\bm W$. It can be viewed as approximating $\mathbf{C}^{\frac{1}{2}}$ only by its diagonal in~\cref{eq:activation_aware_C}. Therefore, the cross-correlation information of the input activations $\bm X$ between each dimension is completely discarded.

Similar to Wanda, in the LLM quantization literature, the state-of-the-art Activation-aware Weight Quantization (AWQ) method~\cite{lin2024awq} computes the scaled $\ell_1$-norm of each row of $\mathbf{X}$, i.e., $\| \mathbf{X}[i,:] \|_1/n$, for $i=1,\ldots,d_\text{in}$, to scale the corresponding column of $\bm W$, and performs quantization on the scaled matrix. 

We also note that Wanda empirically found that, given a target pruning ratio $p$, better performance is achieved with a semi-structured pruning, i.e., by specifically requiring uniform pruning ratio $p$ for each row of $\mathbf{W}_{\text{sparse}}$, instead of restricting sparsity at only the whole matrix level.

The Optimal Brain Compression (OBC) method~\cite{10.5555/3600270.3600593} simplifies the original Optimal Brain Surgeon (OBS)~\cite{lecun1989optimal, hassibi1993optimal} by breaking the compression task into layer-wise sub-problems. For each layer, they use a greedy solver that sequentially prunes (or quantizes) weights. To prune each row of $\bm W$, they iteratively zero out the entry that results in the smallest incremental increase to approximation loss.
However, their follow-up work~\cite{frantar2023sparsegpt} mentions that this method is hard to scale to models with billions of parameters, and further propose several approximations (e.g., prune the weights from left to right and only recalculate approximation loss impact for the weights to the right) for speedup, leading to the SparseGPT method~\cite{frantar2023sparsegpt} for pruning, and the GPTQ~\cite{frantar2022gptq} method for quantization. Similarly, GPFQ~\cite{10.5555/3546258.3546414,doi:10.1137/22M1511709} uses a greedy path-following mechanism to quantize and/or prune weights from left to right, but has theoretical guarantees.

We refer interested readers to~\citet{zhu-etal-2024-survey-model,survey_quantizatoin_book,Nagel2021AWP,kim2023full,wan2023efficient} for a comprehensive overview of weight quantization and pruning methods. Another line of work on structured pruning, e.g.,~\citet{ma2023llm,150fac892d974a75910ecbaf36c65cd0}, removes entire structured components of a network, but usually involves post-training or fine-tuning to recover the performance.

\section{Motivation and Method}
We further decompose \eqref{eq:activation_aware_C} as
\begin{align}\label{eq:row-wise}
\mathcal{L}'(\mathbf{W}_\mathrm{sparse})=& \sum_{i=1}^{d_\text{out}}\|\mathbf{W}[i,:]\mathbf{C}^{\frac{1}{2}}  - \mathbf{W}_{\text{sparse}}[i,:] \mathbf{C}^{\frac{1}{2}} \|_2^2 .
\end{align}
Interestingly, when optimizing under the constraint
\begin{align} \label{eq:row_sparse_C}
\mathcal{C}_{\text{row}} := \big \{ \mathbf \Theta : \forall i \in \{1,\ldots,d_\text{out}\}, \  \|\mathbf \Theta[i,:]\|_0 \leq k \big \},
\end{align}
the problem for each term of \eqref{eq:row-wise} becomes exactly a well-studied sparse approximation problem\footnote{Since $\mathbf{C}^{\frac{1}{2}} $ is a square matrix, this is also a sparse linear regression problem.} of the following general form:
\begin{align}\label{eq:compressive_sensing}
\min_{\boldsymbol \theta} \big [ f(\boldsymbol \theta) &:= \|\mathbf{y}  - \mathbf{A} \boldsymbol \theta \|_2^2 \big ], \\
\text{s.t.} \ \|\boldsymbol{\theta}\|_0 &\leq k := (1-p) \cdot d_\text{in}, \nonumber
\end{align}
where $\mathbf{y}=(\mathbf{W}[i,:]\mathbf{C}^{\frac{1}{2}})^\top$, $\mathbf{A}=[\mathbf{C}^{\frac{1}{2}}]^\top=\mathbf{C}^{\frac{1}{2}} $, and $\boldsymbol \theta$ is the corresponding $\mathbf{W}_{\text{sparse}}[i,:]^\top $ which has $p \cdot d_\text{in}$ zeros that we want to find. $\|\boldsymbol{\theta}\|_0$ is the $\ell_0$ pseudo norm of $\boldsymbol{\theta}$ that counts the total number of nonzero elements in $\boldsymbol{\theta}$.

Many existing sparse approximation and compressive sensing algorithms can be used to solve this problem, such as Matching Pursuit~\cite{1993ITSP...41.3397M}, Orthogonal Matching Pursuit (OMP)~\cite{5895106}, Iterative Hard Thresholding (IHT)~\cite{blumensath2009iterative}, CoSaMP~\cite{tropp2008cosamp}, Basis Pursuit~\cite{chen2001atomic}, Sparse Bayesian Learning methods~\cite{tipping2001sparse,wipf2004sparse}. 
In fact, the OBC method can be viewed as the OMP method in reverse order. 
However, it is well-known in Compressive Sensing that such methods are too greedy and often outperformed by IHT and $\ell_1$ type methods.

We adopt the IHT method, since it is efficient to run and also has \textit{recovery guarantees}~\cite{blumensath2009iterative}. IHT is a special form of the projected gradient descent (PGD) method. Each iteration of PGD involves calculating the gradient\footnote{Using the fact that $\mathbf{y}=(\mathbf{W}[i,:]\mathbf{C}^{\frac{1}{2}})^\top$ can simplify the gradient computation without doing SVD of $\mathbf{C}$, as explained later.} of $f(\boldsymbol \theta)$, i.e., $\nabla f(\boldsymbol \theta^{(t)})$, and projecting the updated weights $(\boldsymbol \theta^{(t)}-\eta \nabla f(\boldsymbol \theta^{(t)}))$ onto the constraint set.
For IHT, the projection onto the constraint $\|\boldsymbol \theta\|_0 \leq k$ is simply hard-thresholding, i.e., keeping the $k$ largest-magnitude elements of ($\boldsymbol \theta^{(t)}-\eta \nabla f(\boldsymbol \theta^{(t)})$, and setting the remaining elements to $0$.
Actually all the terms in \eqref{eq:row-wise} can be solved independently and in parallel. This can also be viewed as minimizing the following objective, with the constraint of \eqref{eq:row_sparse_C} that each row of $\hat{\mathbf{W}}$ is $k$-sparse:
\begin{align}\label{eq:objective_activation_aware}
\min_{\hat{\mathbf{W}} \in \mathcal{C}_{\text{row}}}
\big [ f_1(\hat{\mathbf{W}}) := \|\mathbf{W}\mathbf{C}^{\frac{1}{2}}  - \hat{\mathbf{W}}\mathbf{C}^{\frac{1}{2}} \|_\mathrm{F}^2 \big ].
\end{align}
For this general objective $f_1$ in \eqref{eq:objective_activation_aware}, its gradient w.r.t. $\hat{\mathbf{W}}$ is:
\begin{align}
\nabla f_1(\hat{\mathbf{W}})=&-2  (\mathbf{W}\mathbf{C}^{\frac{1}{2}}  - \hat{\mathbf{W}}\mathbf{C}^{\frac{1}{2}})(\mathbf{C}^{\frac{1}{2}})^\top \\
= & -2  (\mathbf{W}  - \hat{\mathbf{W}})\mathbf{C}.
\label{eq:gradient}
\end{align}
Fortunately, even though the objective in \cref{eq:objective_activation_aware} has $\mathbf{C}^{\frac{1}{2}}$, the actual gradient calculation in \eqref{eq:gradient} only has $\bm C$, i.e., we can avoid calculating $\mathbf{C}^{\frac{1}{2}}$ and its expensive SVD computation.

\begin{algorithm}[t]
\small
   \caption{Activation-Aware Projected Gradient Descent}
   \label{alg}
\begin{algorithmic}
   \STATE {\bfseries Input:} original weight  $\bm W\in \mathbb{R}^{d_\text{out} \times d_\text{in}}$, input activation covariance $\mathbf{C}=\frac{1}{n}\bm X \bm X^\top$, constraints $\mathcal{C}$, step size $\eta$
   \STATE {\bfseries Initialize: $\mathbf{\Theta}^{(0)} \in \mathcal{C}$}
   \STATE {\bfseries Repeat}\\
    \quad $\mathbf{Z}^{(t)}=\mathbf{\Theta}^{(t)} + \eta  (\bm W-\mathbf{\Theta}^{(t)} )\bm C$;\\
    \quad $\mathbf{\Theta}^{(t+1)}=\mathsf{Proj}_{\mathcal{C}}(\mathbf{Z}^{(t)})$;
    \STATE {\bfseries until} a stopping criterion is met
    \STATE {\bfseries Output} compressed weight $\mathbf{\Theta}$
\end{algorithmic}

\end{algorithm}
The overall activation-aware PGD method for compressing the weight is described in \cref{alg}, which we refer to as \textbf{A}ctivation-aware \textbf{W}eight pruning and quantization via \textbf{P}GD (AWP).
For semi-structured pruning, i.e., where each row is $k$-sparse, with $\mathcal{C}_{\text{row}}$ as given by \eqref{eq:row_sparse_C}, $\mathsf{Proj}_{\mathcal{C}_{\text{row}}}(\mathbf{Z})$ simply keeps the $k$ largest-magnitude elements in each row of $\mathbf{Z}$ and sets the remaining elements to 0.
While for quantization, e.g., INT4 quantization, $\mathsf{Proj}_{\mathcal{C}_{\text{INT4}}} (\mathbf{Z})$ would quantize $\mathbf{Z}$ into INT4.

Furthermore, \cref{alg} allows joint pruning and quantization, where the constraint set $\mathcal C$ becomes the intersection of $\mathcal{C}_{\text{sparse}}$ and ${\mathcal{C}_{\text{INT4}}} $. One can either prune the quantized version of $\mathbf{Z}$, or first prune $\mathbf{Z}$ and obtain the corresponding sparsity mask, then quantize the pruned version and finally apply the sparsity mask.   

Note that the main computational cost of \cref{alg} is the gradient descent, which involves multiplication between $ (\bm W-\mathbf{\Theta}^{(t)} )$ and $\bm C$, incurring $O(d_\text{out} \times d_\text{in}^2)$. This is computationally more efficient than inverting $\bm X \bm X^\top$ required in OBC, SparseGPT, GPTQ, etc., and can be run in parallel on the GPU.

\section{Experiments}
We compare the proposed AWP method with state-of-the-art methods on pruning, quantization, as well as joint pruning and quantization.
As in AWQ and Wanda, we evaluate the perplexity of the compressed model on the held-out WikiText-2~\cite{merity2016pointer} validation set.

\subsection{Pruning}
To compare with Wanda, SparseGPT, and Magnitude-based pruning, we follow the exact setup of the experiments in Wanda\footnote{https://github.com/locuslab/wanda}. More specifically, we tested on the Llama-2 7B and 13B models. The calibration data $\bm X$ is 128 sequences (each has 2048 tokens) sampled from the C4 training set~\cite{raffel2020exploring}. 
We test pruning ratios from $\{50\%, 60\%, 70\%, 80\%, 90\%\}$. Note that larger pruning ratio corresponds to higher compression rate. 
For AWP, the iteration stops when the Frobenius norm of the gradient normalized by the Frobenius norm of the original weight is less than 0.0001, or 200 iterations is reached. The step size $\eta$ is set as $2/ \|\mathbf{C}\|_\mathrm{F}$. As problem \eqref{eq:activation_aware_C} is nonconvex, we initialize  $\mathbf{\Theta}^{(0)}$ as the solution of Wanda, since a good initial point helps nonconvex optimization. 
 
\cref{table:prune7B} shows the perplexity of the pruned Llama-2-7B model for different methods and pruning ratios, and \cref{table:prune13B} shows the corresponding results for the Llama-2-13B model. The results of SparseGPT and Magnitude-based pruning are from~\citet{sun2023wanda}. As a reference, the perplexity of the original Llama-2-7B dense model (i.e., pruning ratio = 0\%) is 5.12 and the perplexity of the original Llama-2-13B dense model is 4.57.

\begin{table}[h]
\vskip -0.15in
\caption{Perplexity on WikiText2 of pruned Llama-2-7B model by different methods under different pruning ratios. }
\label{table:prune7B}
\begin{center}
\begin{small}
\begin{sc}
\begin{tabular}{lcccccr}
\toprule
&  $50\%$ & $60\%$ & $70\%$ & $80\%$& 90\%  \\
\midrule
Magnitude & $14.89$& $4e3$& - & NaN & - \\ 
SparseGPT & $6.51$ & $9.58$ & - & $1e2$ & - \\
Wanda & $6.48$ & $10.09$ & $70.04$ & $4e3$ & $1e4$ \\
\textbf{AWP} & $\textbf{6.42}$ & $\textbf{9.44}$ & $\textbf{22.10}$ & $\textbf{83.28}$ & $\mathbf{8e2}$
\\
\bottomrule
\end{tabular}
\end{sc}
\end{small}
\end{center}
\vskip -0.27in
\end{table}

\begin{table}[h]
\caption{Perplexity on WikiText2 of pruned Llama-2-13B model by different methods under different pruning ratios. }
\label{table:prune13B}
\vspace{-1pt}
\begin{center}
\begin{small}
\begin{sc}
\begin{tabular}{lcccccr}
\toprule
&  $50\%$ & $60\%$ & $70\%$ & $80\%$& 90\%  \\
\midrule
Magnitude & $6.37$& $11.23$& - & $5e4$ & - \\ 
SparseGPT & $5.63$ & $7.80$ & - & $1e2$ & - \\
Wanda & $5.59$ & $7.97$ & $43.06$ & $1e3$ & $2e4$ \\
\textbf{AWP} & $\textbf{5.54}$ & $\textbf{7.49}$ & $\textbf{16.57}$ & $\textbf{75.68}$ & $\mathbf{1e3}$
\\
\bottomrule
\end{tabular}
\end{sc}
\end{small}
\end{center}
\vskip -0.17in
\end{table}
\vspace{0mm}
We can see that for all methods, the perplexity gets worse when the pruning ratio increases. Nonetheless, the activation-aware methods, i.e., SparseGPT, Wanda, and proposed AWP perform significantly better than the magnitude based pruning, which is not activation-aware. The proposed AWP method has the best perplexity under all pruning ratios, especially when the pruning ratio is over 60\%.

\vskip -0.1in
\subsection{Quantization}
\vskip -0.03in
We compare with state-of-the-art layer-wise weight quantization methods AWQ\footnote{https://github.com/mit-han-lab/llm-awq} and GPTQ\footnote{https://github.com/AutoGPTQ/AutoGPTQ/tree/main}. We follow the experiment setups of AWQ, which focus on weight-only grouped quantization with a group size of 128. As in AWQ, we use a small calibration set from the Pile dataset~\cite{gao2020pile} in order not to overfit to a specific downstream domain. Besides INT4 and INT3 quantization, we additionally experiment with INT2 quantization.

For the proposed AWP method, the step size $\eta$ is set to $1.5/ \|\mathbf{C}\|_\mathrm{F}$, and we only run 10 iterations. We simply initialize $\mathbf{\Theta}^{(0)}$ to be the straightforward (i.e., not activation-aware) Round-To-Nearest (RTN) quantized version of $\bm W$. 

\cref{table:quantize8B} shows the perplexity of the INT4/INT3/INT2 quantized Llama-3.1-8B model by GPTQ, AWQ, and our proposed AWP method. AWP quantized INT4 and INT3 models have better perplexity than corresponding INT4 and INT3 models quantized by GPTQ and AWQ. For INT2 quantization, the resulting perplexities of all methods are very poor, although GPTQ has the lowest perplexity.

\begin{table}[h]
\vskip -0.19in
\caption{Perplexity on WikiText2 of quantized Llama-3.1-8B model by different methods. }
\label{table:quantize8B}
\begin{center}
\begin{small}
\begin{sc}
\begin{tabular}{lcccr}
\toprule
&  INT4 & INT3 & INT2   \\
\midrule
GPTQ & $9.95$ & 12.54 & $\mathbf {2e3}$\\
AWQ & $6.64$ & $8.14$ & $ {3e4}$ \\
\textbf{AWP} & $\textbf{6.55}$ & $\textbf{8.06}$ & ${1e6}$ 
\\
\bottomrule
\end{tabular}
\end{sc}
\end{small}
\end{center}
\vskip -0.2in
\end{table}

\subsection{Joint Pruning and Quantization}
Note that, as shown in \cref{table:quantize8B}, INT4 quantization by AWP achieves good perplexity (much better than INT2). Further, \cref{table:prune7B} and \cref{table:prune13B} show that pruning the original model up to 70\% has moderate perplexity degradation using the proposed AWP method. This motivates us to consider combining pruning with quantization to further compress the model, instead of solely pushing the quantization to extremely low-bits. 

Note that the proposed AWP method naturally allows joint pruning and quantization. We further compare it with sequential quantization and pruning using state-of-the-art methods AWQ+Wanda, as well as sequential pruning and quantization using Wanda+AWQ. We use INT4 quantization for all methods and test pruning ratios of 25\%, 50\%, and 75\%.

For AWP, we fix the step size $\eta=1.5/ \|\mathbf{C}\|_\mathrm{F}$ as in the quantization setting. Instead of directly compressing the model into the target bits and pruning ratio, we first gradually increase the pruning ratio without quantization for 50 iterations, then perform 50 joint pruning and INT4 quantization iterations. More specifically, in the first 25 iterations of purely pruning, we linearly increase the pruning ratio from 0\% to the target pruning ratio, then keep this pruning ratio unchanged in the remaining 75 iterations. In each iteration of our joint pruning and INT4 quantization, the projection is $\mathsf{Proj}_{\mathcal{C}_{\text{INT4}}}(\mathsf{Proj}_{\mathcal{C}_\text{row}} (\mathbf{Z}))$. At the end of iterations, the corresponding sparsity mask is applied to ensure that the final weight is both sparsified and quantized. 

We first test on the Llama-3.1-8B model. \cref{table:joint8B} shows the perplexity of pruned and INT4 quantized models by different methods under different pruning ratios. We can see that Wanda+AWQ (pruning first) is consistently better than AWQ+Wanda (quantization first) under different pruning ratios. The proposed AWP achieves the best performance under all pruning ratios, especially when the pruning ratio is high. 

An interesting finding is that, compared with INT2 quantization in \cref{table:quantize8B}, INT4 quantization with 75\% pruning ratio achieves significantly better perplexity. Note that 4 bits combined with 75\% pruning rate is roughly equivalent to 2 bits, as we need 1 bit to store the pruning mask. As recent efforts try to push the quantization to extremely low bits, e.g., 1.58 bit~\cite{ma2024era,wang20241}, our results show that combining quantization with pruning may achieve much better performance.

\begin{table}[t]
\vskip -0.15in
\caption{Perplexity on WikiText2 of pruned and INT4 quantized Llama-3.1-8B model by different methods. }
\label{table:joint8B}
\begin{center}
\begin{small}
\begin{sc}
\begin{tabular}{lcccr}
\toprule
  Pruning Ratio:& 25\%  & 50\% & 75\%   \\
\midrule
AWQ+Wanda & $6.93$ & $9.71$ & $ {3e2}$ \\
Wanda+AWQ & $\textbf{6.81}$ & $9.46$ & $ {2e2}$ \\
\textbf{AWP} & $\textbf{6.81}$ & $\textbf{9.32}$ & $\bm {1e2}$ 
\\
\bottomrule
\end{tabular}
\end{sc}
\end{small}
\end{center}
\vskip -0.25in
\end{table}

\begin{table}[t]
\caption{Perplexity on WikiText2 of pruned and INT4 quantized Llama-3.2-1B model by different methods. }
\label{table:joint1B}
\vspace{-2pt}
\begin{center}
\begin{small}
\begin{sc}
\begin{tabular}{lcccr}
\toprule
 Pruning Ratio:& 25\%  & 50\% & 75\%   \\
\midrule
AWQ+Wanda & $11.63$ & $23.95$ & $ {2e3}$ \\
Wanda+AWQ & $11.30$ & $21.90$ & $ {1e3}$ \\
\textbf{AWP} & $\textbf{11.20}$ & $\textbf{18.41}$ & $\bm {3e2}$ 
\\
\bottomrule
\end{tabular}
\end{sc}
\end{small}
\end{center}
\vskip -0.2in
\end{table}

\section{Conclusion and Discussion}

We have proposed a layer-wise activation-aware post-training quantization and pruning method based on PGD. 
We provided a new insight from the compressive sensing framework to compress large foundation models.
Empirical studies show that the proposed method outperforms the state-of-the-art LLM pruning and quantization methods. We hope our work will inspire more advanced LLM compression methods by leveraging cutting-edge compressive sensing techniques.

Our future directions include extending pruning to structured sparsity, e.g., NVIDIA's 2:4 sparsity~\cite{mishra2021accelerating}.
Another direction is to provide theoretical guarantees of the proposed method for quantization.

\bibliography{example_paper}
\bibliographystyle{icml2025}

\newpage
\appendix
\onecolumn

\section{Theoretical Justifications of AWP for Pruning}

\subsection{Convergence under Restricted Isometry Property (RIP)}
We first recall the theoretical guarantees of IHT algorithm~\cite{blumensath2009iterative}, which uses the iteration:
$${\boldsymbol \theta}^{(t+1)}=H_k({\boldsymbol \theta}^{(t)} + \bm A^{\top}(\bm y-\bm A{\boldsymbol \theta}^{(t)} ))$$
where $H_k(\cdot)$ is the hard-thresholding operator that keeps $k$ largest-magnitude elements of the input vector and sets the remaining elements to 0.

\begin{theorem}
~[\citet{blumensath2009iterative}] Given a noisy observation $\mathbf{y}  = \mathbf{A} \boldsymbol \theta_k +\bm e$, where $\boldsymbol \theta_k$ is $k$-sparse. If $\bm A$ has the restricted isometry property with $\beta_{3k}<1/8$, then, at iteration $t$, IHT will recover an approximation $\boldsymbol \theta^{(t)}$ satisfying
$$\| \boldsymbol \theta^{(t)}-\boldsymbol \theta_k \|_2 \leq \|\boldsymbol \theta_k\|_2/2^t + 4 \|\bm e\|_2.$$

Furthermore, after at most 
$t'= \lceil \log_2(\|\boldsymbol \theta_k \|_2 / \|\bm e\|_2)\rceil$ iterations,
$$\| \boldsymbol \theta^{(t')}-\boldsymbol \theta_k \|_2 \leq 5 \|\bm e\|_2.$$

\end{theorem}
The definition of $\beta_{3k}$ in RIP can be found in Equation (6) of \citet{blumensath2009iterative}.

Note that running AWP Algorithm 1 for pruning with $\mathcal{C}_{\text{row}}=\{\mathbf \Theta : \|\mathbf \Theta[i,:]\|_0 \leq k, i=1,...,d_\mathrm{out}\}$ and step size $\eta=1$ is equivalent to running IHT for each row of the weight matrix in parallel. Recall that in our semi-structured pruning (i.e., each row is targeted to be $k$-sparse), $\mathcal{C}_{\text{row}}=\{\mathbf \Theta : \|\mathbf \Theta[i,:]\|_0 \leq k, i=1,...,d_\mathrm{out}\}$ and $\mathsf{Proj}_{\mathcal{C}_{\text{row}}}(\mathbf{Z})$ simply keeps $k$ largest-magnitude elements in each row of $\mathbf{Z}$ and set remaining elements to 0. 

More specifically, recall that each term in \eqref{eq:row-wise} is:
 $$\|\mathbf{W}[i,:]\mathbf{C}^{\frac{1}{2}}  - \mathbf{W}_{\text{sparse}}[i,:] \mathbf{C}^{\frac{1}{2}} \|_2^2=\|\underbrace{(\mathbf{W}[i,:]\mathbf{C}^{\frac{1}{2}})^\top}_{\bm y}  - \underbrace{(\mathbf{C}^{\frac{1}{2}})^\top}_{\bm A} \underbrace{\mathbf{W}_{\text{sparse}}[i,:]^\top}_{\boldsymbol \theta}  \|_2^2. $$
We can view 
$\mathbf{y}=(\mathbf{W}[i,:]\mathbf{C}^{\frac{1}{2}})^\top$, $\mathbf{A}=(\mathbf{C}^{\frac{1}{2}})^\top$, and $\mathbf{W}_{\text{sparse}}[i,:]^\top$ corresponds to $\boldsymbol \theta$. Let us denote the \textbf{global optimal} $k$-sparse solution $\mathbf{W}_{\text{sparse}}[i,:]$ to the above problem as $\mathbf{W}^*_{\text{sparse}}[i,:]$, which corresponds to $\boldsymbol \theta_k$, then we have the corresponding optimal $k$-sparse approximation error $\bm e = (\mathbf{W}[i,:]\mathbf{C}^{\frac{1}{2}}  - \mathbf{W}_{\text{sparse}}^*[i,:] \mathbf{C}^{\frac{1}{2}})^\top$, which is treated as a noise term.

Therefore, we have following guarantee of AWP pruning for each row:

\begin{theorem}\label{theorem:awq_row}
If $\mathbf{C}^{\frac{1}{2}}$ has restricted isometry property with $\beta_{3k}<1/8$, Algorithm 1 with $\mathcal{C}_{\textrm{row}}=\{\mathbf \Theta : \|\mathbf \Theta[i,:]\|_0 \leq k, i=1,...,d_\textrm{out}\}$ and $\eta =1$ will recover an approximation $\mathbf{\Theta}[i,:]^{(t)}$ satisfying
$$\| \mathbf{\Theta}[i,:]^{(t)} -\mathbf{W}^*_{\text{sparse}}[i,:] \|_2 \leq \|\mathbf{W}^*_{\text{sparse}}[i,:]\|_2/2^t + 4 \|\mathbf{W}[i,:]\mathbf{C}^{\frac{1}{2}}  - \mathbf{W}_{\text{sparse}}^*[i,:] \mathbf{C}^{\frac{1}{2}}\|_2.$$

Furthermore, after at most 
$t'= \max_i \lceil \log_2(\|\mathbf{W}^*_{\text{sparse}}[i,:] \|_2 / \|(\mathbf{W}[i,:]\mathbf{C}^{\frac{1}{2}}  - \mathbf{W}_{\text{sparse}}^*[i,:] \mathbf{C}^{\frac{1}{2}})^\top\|_2)\rceil$ iterations,
$$\| \mathbf{\Theta}[i,:]^{(t')}-\mathbf{W}_{\text{sparse}}^*[i,:]  \|_2 \leq 5 \|\mathbf{W}[i,:]\mathbf{C}^{\frac{1}{2}}  - \mathbf{W}_{\text{sparse}}^*[i,:] \mathbf{C}^{\frac{1}{2}}\|_2.$$

\end{theorem}

Note that one can always scale the matrix $\mathbf{C}^{\frac{1}{2}}$ in the activation-aware objective function \eqref{eq:activation_aware_C}.

Combining the error bounds from all rows, we have following corollary:

\begin{corollary}\label{corollary}
If $\mathbf{C}^{\frac{1}{2}}$ has restricted isometry property with $\beta_{3k}<1/8$, running Algorithm 1 with $\mathcal{C}_{\text{row}}=\{\mathbf \Theta : \|\mathbf \Theta[i,:]\|_0 \leq k, i=1,...,d_\textrm{out}\}$ and $\eta =1$ after at most 
$t'= \max_i \lceil \log_2(\|\mathbf{W}_{\text{sparse}}[i,:]^* \|_2 / \|(\mathbf{W}[i,:]\mathbf{C}^{\frac{1}{2}}  - \mathbf{W}_{\text{sparse}}^*[i,:] \mathbf{C}^{\frac{1}{2}})^\top\|_2)\rceil$ iterations, we have 
\begin{equation}\label{eq:bound}
\| \mathbf{\Theta}^{(t')}-\mathbf{W}_{\text{sparse}}^* \|_\mathrm{F} \leq 5 \|\mathbf{W}\mathbf{C}^{\frac{1}{2}}  - \mathbf{W}_{\text{sparse}}^* \mathbf{C}^{\frac{1}{2}}\|_\mathrm{F}.    
\end{equation}

\end{corollary}
\begin{proof}
From \cref{theorem:awq_row}, we have
$$\| \mathbf{\Theta}[i,:]^{(t')}-\mathbf{W}_{\text{sparse}}^*[i,:]  \|_2 \leq 5 \|\mathbf{W}[i,:]\mathbf{C}^{\frac{1}{2}}  - \mathbf{W}_{\text{sparse}}^*[i,:] \mathbf{C}^{\frac{1}{2}}\|_2.$$
Squaring both sides, we have
$$\| \mathbf{\Theta}[i,:]^{(t')}-\mathbf{W}_{\text{sparse}}^*[i,:]  \|_2^2 \leq 25 \|\mathbf{W}[i,:]\mathbf{C}^{\frac{1}{2}}  - \mathbf{W}_{\text{sparse}}^*[i,:] \mathbf{C}^{\frac{1}{2}}\|_2^2.$$
Summing over all rows, we have
$$\sum_i^{d_\text{out}}\| \mathbf{\Theta}[i,:]^{(t')}-\mathbf{W}_{\text{sparse}}^*[i,:]  \|_2^2 \leq 25\sum_i^{d_\text{out}} \|\mathbf{W}[i,:]\mathbf{C}^{\frac{1}{2}}  - \mathbf{W}_{\text{sparse}}^*[i,:] \mathbf{C}^{\frac{1}{2}}\|_2^2.$$
So we have
$$\| \mathbf{\Theta}^{(t')}-\mathbf{W}_{\text{sparse}}^* \|_\mathrm{F}^2 \leq 25 \|\mathbf{W}\mathbf{C}^{\frac{1}{2}}  - \mathbf{W}_{\text{sparse}}^* \mathbf{C}^{\frac{1}{2}}\|_\mathrm{F}^2.$$
Finally, taking square root on both sides, we obtain~\cref{eq:bound}.
\end{proof}

\subsection{Convergence under Restricted Strong Convexity (RSC)
and Restricted Smoothness (RSM) Property}
On the other hand, instead of requiring the restricted isometry property (RIP) of $\mathbf{C}^{\frac{1}{2}}$, which is not easy to verify~\cite{wang2016average}, \citet{10.5555/2968826.2968903} provided the recovery guarantee of IHT method based on Restricted Strong Convexity (RSC) and Restricted Smoothness (RSM) properties defined below~\cite{10.5555/2968826.2968903,10.1093/imaiai/iaz027}:

\begin{definition}\label{def:RSC} (RSC Property) A differential function $f:\mathbb{R}^{d} \rightarrow \mathbb{R}$ satisfies restricted strong convexity with parameter $\alpha$ at sparsity level $k$, abbreviated as $(\alpha, k)$-RSC, if the following holds for all ${\boldsymbol \theta}_1, {\boldsymbol \theta}_2 \ s.t. \ \|{\boldsymbol \theta}_1\|_0 \leq k $ and $\|{\boldsymbol \theta}_2\|_0 \leq k$:
$$f({\boldsymbol \theta}_1)\geq f({\boldsymbol \theta}_2)+\langle\nabla_{\boldsymbol \theta} f(\boldsymbol \theta_2),\boldsymbol \theta_1-\boldsymbol \theta_2\rangle+\frac{\alpha}{2}\|\boldsymbol \theta_1-\boldsymbol \theta_2\|_2^2.$$
\end{definition}
\begin{definition}\label{def:RSM}  (RSM Property) A differential function $f:\mathbb{R}^{d} \rightarrow \mathbb{R}$ satisfies restricted strong smoothness with parameter $\beta$ at sparsity level $k$, abbreviated as $(\beta, k)$-RSM, if the following holds for all ${\boldsymbol \theta}_1, {\boldsymbol \theta}_2 \ s.t. \ \|{\boldsymbol \theta}_1\|_0 \leq k $ and $\|{\boldsymbol \theta}_2\|_0 \leq k$:
$$f({\boldsymbol \theta}_1)\leq f({\boldsymbol \theta}_2)+\langle\nabla_{\boldsymbol \theta} f(\boldsymbol \theta_2),\boldsymbol \theta_1-\boldsymbol \theta_2\rangle+\frac{\beta}{2}\|\boldsymbol \theta_1-\boldsymbol \theta_2\|_2^2.$$
\end{definition}
Based on the above two properties, ~\citet{10.1093/imaiai/iaz027} states that for an objective function $f$ satisfying $(\alpha, k)$-RSC and $(\beta, k)$-RSM, IHT with step size $\eta \propto 1/\beta$, i.e., ${\boldsymbol \theta}^{(t+1)}=H_k({\boldsymbol \theta}^{(t)} - \eta \nabla_{\boldsymbol \theta} f({\boldsymbol \theta}^{(t)} ))$, satisfies
$$f(\boldsymbol \theta^{(t)}) \leq \min_{\|\boldsymbol \theta\|_0 \leq k/(32\kappa^2)}\big [f(\boldsymbol \theta) +\big(1-\frac{1}{12\kappa}\big)^t \cdot \big (f(\boldsymbol \theta^{(0)})-f(\boldsymbol \theta)\big ) \big ], $$
where $\kappa=\beta / \alpha$, known as the condition number of Hessian of $f$. 
In other words, it shows linear convergence to the bound
$$\lim_{t \rightarrow \infty} f(\boldsymbol \theta^{(t)})\leq \min_{\|\boldsymbol \theta\|_0 \leq k/(32\kappa^2)} f(\boldsymbol \theta).$$

For our quadratic objective $f(\boldsymbol \theta) := \|\mathbf{y}  - \mathbf{A} \boldsymbol \theta \|_2^2 $ in \cref{eq:compressive_sensing}, we have exactly
\begin{align}
  f({\boldsymbol \theta}_1)&= f({\boldsymbol \theta}_2)+\langle\nabla_{\boldsymbol \theta} f(\boldsymbol \theta_2),\boldsymbol \theta_1-\boldsymbol \theta_2\rangle+(\boldsymbol \theta_1-\boldsymbol \theta_2)^\top\frac{H_f(\boldsymbol \theta_2)}{2}(\boldsymbol \theta_1-\boldsymbol \theta_2)  \\
  &= f({\boldsymbol \theta}_2)+\langle\nabla_{\boldsymbol \theta} f(\boldsymbol \theta_2),\boldsymbol \theta_1-\boldsymbol \theta_2\rangle+(\boldsymbol \theta_1-\boldsymbol \theta_2)^\top \frac{2\bm A^\top \bm A}{2}(\boldsymbol \theta_1-\boldsymbol \theta_2) \\
  &= f({\boldsymbol \theta}_2)+\langle\nabla_{\boldsymbol \theta} f(\boldsymbol \theta_2),\boldsymbol \theta_1-\boldsymbol \theta_2\rangle+(\boldsymbol \theta_1-\boldsymbol \theta_2)^\top\frac{2\bm C}{2}(\boldsymbol \theta_1-\boldsymbol \theta_2) .
\end{align}
Therefore, one can simply use the smallest singular value of $2\bm C$, i.e., $2\lambda_{\min} (\bm C) $ as $\alpha$, and largest singular value of $2\bm C$, i.e., $2\lambda_{\max} (\bm C) $ as $\beta$. 
In other words, our objective \eqref{eq:compressive_sensing} satisfies $(2\lambda_{\min}(\bm C), k)$-RSC and $(2\lambda_{\max}(\bm C), k)$-RSM, and AWP pruning inherits above convergence and recovery guarantees with $\kappa=\lambda_{\max}(\bm C) /\lambda_{\min}(\bm C) $, i.e., the condition number of $\bm C$.

\begin{remark}
Note that, in \cref{def:RSC} and \cref{def:RSM}, $ \ \|{\boldsymbol \theta}_1\|_0 \leq k $ and $\|{\boldsymbol \theta}_2\|_0 \leq k$, therefore $\|{\boldsymbol \theta}_1-{\boldsymbol \theta}_2\|_0 \leq \min(2k,d_\mathrm{in})$, which is still a sparse vector when $k<d_\mathrm{in}/2$. So $2\lambda_{\min}(\bm C)$ and $2\lambda_{\max}(\bm C)$ are usually loose lower and upper bounds for $\alpha$ and $\beta$, respectively, and $\lambda_{\max}(\bm C) /\lambda_{\min}(\bm C)$ is usually a loose upper bound for $\kappa$. 
One can see that, the smaller the $\kappa$, the better the convergence guarantee and recovery guarantee.
\end{remark}

\section{Activation-Aware Loss Derivation}

The expression for the activation-aware loss in \cref{eq:activation_aware_C} is derived as follows:
\begin{align*}
\| \mathbf{W}\mathbf{X}-\mathbf{W}_{\text{sparse}}\mathbf{X} \|_\mathrm{F}^2
&= \mathrm{tr}\big[ 
(\mathbf W-\mathbf{W}_{\text{sparse}}) (\mathbf X \mathbf X^\top) (\mathbf W-\mathbf{W}_{\text{sparse}})^\top
\big] \\
&= \mathrm{tr}\big[ 
(\mathbf W-\mathbf{W}_{\text{sparse}}) (\mathbf{XX}^\top)^{1/2}(\mathbf{XX}^\top)^{1/2} (\mathbf W-\mathbf{W}_{\text{sparse}})^\top
\big] \\
&= \| (\mathbf W-\mathbf{W}_{\text{sparse}}) (\mathbf{XX}^\top)^{1/2}\|_\mathrm{F}^2\\   
&= \| \mathbf W (\mathbf{XX}^\top)^{1/2}-\mathbf{W}_{\text{sparse}}(\mathbf{XX}^\top)^{1/2}\|_\mathrm{F}^2.
\end{align*}

\section{Activation-aware Loss w.r.t. Iterations}
\cref{fig:iterations} shows an example of normalized activation-aware loss \eqref{eq:objective_activation_aware}, i.e., $\| \mathbf{W}\mathbf{C}^{\frac{1}{2}} -\mathbf{\Theta}^{(t)}\mathbf{C}^{\frac{1}{2}}\|_\mathrm{F}/\| \mathbf{W}\|_\mathrm{F}$, w.r.t. iteration $t$ during AWP pruning of a layer in the Llama-2 7B model. We can see that such activation-aware approximation loss is effectively minimized by the proposed activation-aware projected gradient descent method described in \cref{alg}.
\begin{figure}[h!]
    \centering
    \includegraphics[width=0.8\textwidth]{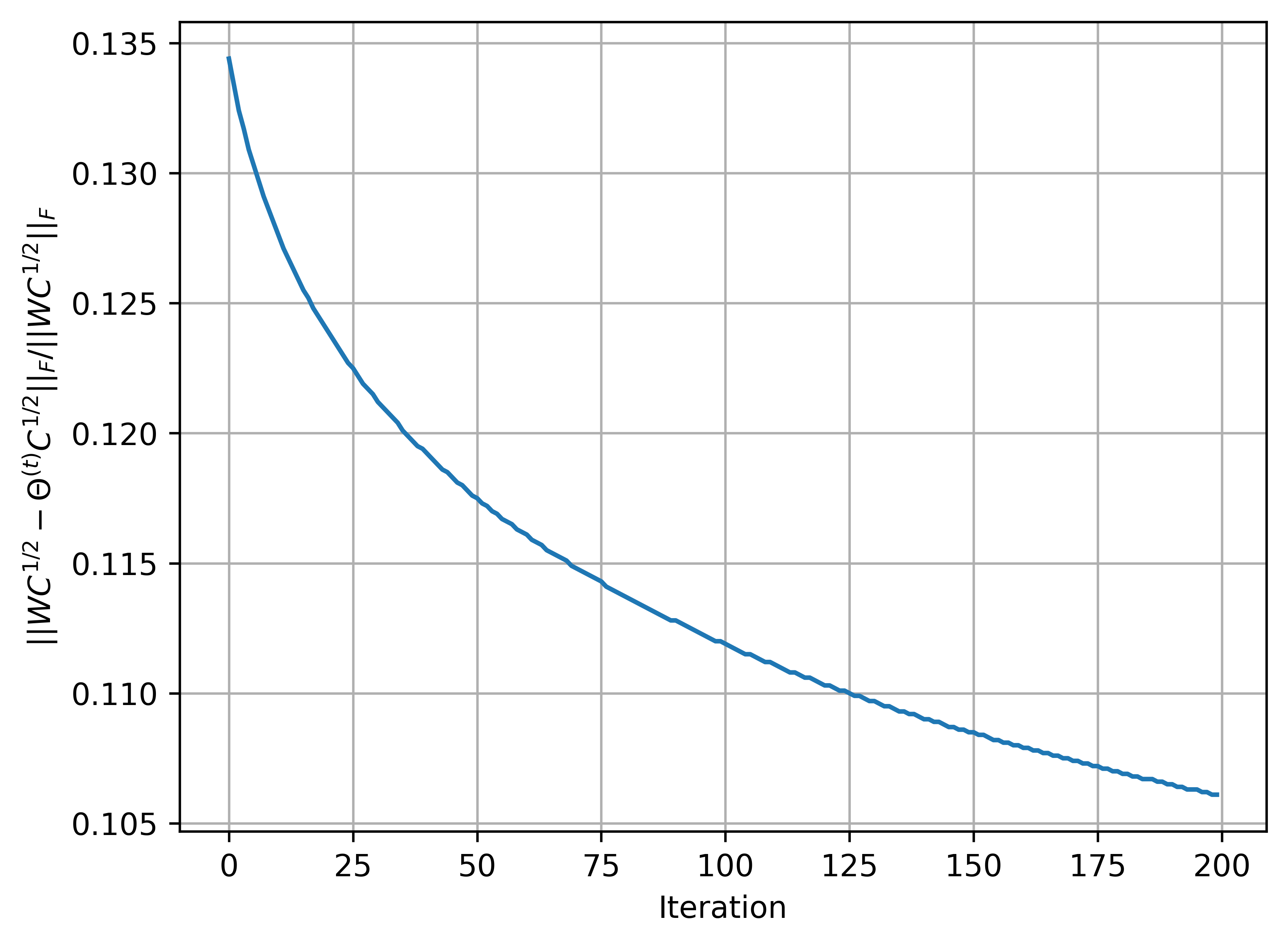} %
    \caption{$\| \mathbf{W}\mathbf{C}^{\frac{1}{2}} -\mathbf{\Theta}^{(t)}\mathbf{C}^{\frac{1}{2}}\|_\mathrm{F}/\| \mathbf{W}\|_\mathrm{F}$, w.r.t. iteration $t$ during AWP pruning of a layer in the Llama-2 7B model.}
    \label{fig:iterations}
\end{figure}
\end{document}